\pdfoutput=1
\documentclass[journal]{IEEEtran}
\usepackage{amsmath,amssymb,amsthm}
\usepackage{graphicx}
\usepackage{cite}
\usepackage{booktabs}
\usepackage{array}
\usepackage{xcolor}
\usepackage{flafter}%

\newtheorem{theorem}{Theorem}
\newtheorem{proposition}{Proposition}

\newtheorem{lemma}{Lemma}

\theoremstyle{remark}

\newcommand{\X}{\mathcal{X}}
\newcommand{\Xt}{\widetilde{\mathcal{X}}}
\newcommand{\E}{\mathcal{E}}
\newcommand{\G}{\mathcal{G}}

\newcommand{\bP}{\mathbf{P}}
\newcommand{\bQ}{\mathbf{Q}}
\newcommand{\bD}{\mathbf{D}}
\newcommand{\bU}{\mathbf{U}}
\newcommand{\bI}{\mathbf{I}}
\newcommand{\KK}{\mathbf{K}}
\newcommand{\Id}{\mathbf{I}}
\newcommand{\vx}{\mathbf{x}}
\newcommand{\xt}{\tilde{\mathbf{x}}}
\newcommand{\ve}{\mathbf{e}}
\newcommand{\vd}{\mathbf{d}}
\newcommand{\vg}{\mathbf{g}}
\newcommand{\bPi}{\boldsymbol{\Pi}}
\newcommand{\RE}{\mathrm{RE}}
\newcommand{\LC}{L_{\mathrm{C}}}
\newcommand{\LMI}{L_{\mathrm{MI}}}
\newcommand{\LD}{L_{\mathrm{D}}}
\newcommand{\LO}{L_{\mathrm{O}}}
\newcommand{\F}{\mathcal{F}}
\newcommand{\nrm}[1]{\|#1\|}
\DeclareMathOperator{\rank}{rank}
\newcommand{\tabinput}[1]{\csname @@input\endcsname #1 }

\begin{document}

\title{How Vector, Matrix, and Tensor Representations Shape Reconstruction Error}

\author{Hiroki~Hasegawa%
\thanks{This work was supported by JST SPRING, Grant Number \mbox{JPMJSP2124}.}%
\thanks{H. Hasegawa is with the University of Tsukuba, Graduate School of Science and Technology, Tsukuba, 305-8573, Japan (e-mail: hasegawa.hiroki.tkb\_en@u.tsukuba.ac.jp; ORCID: 0009-0009-3022-4359).}}

\markboth{Manuscript under review}{Hasegawa: How Vector, Matrix, and Tensor Representations Shape Reconstruction Error}

\maketitle

\begin{abstract}
Multidimensional signals, compressed as vectors, matrices, or tensors by low-rank approximation, are judged by the reconstruction error, a single value that cannot reveal what was lost. We ask how the representation shapes this error. Magnitude: the representation sets the allowed approximations and hence the least achievable error, not the error of a given reconstruction; tighter nested constraints add exactly new loss. Composition: at equal magnitude, even for one reconstruction, the error composition can differ with the representation; a truncated unfolding loses only directions as a matrix but only interaction as a tensor. Both follow from one orthogonal decomposition of the squared error, generalized to any order, into coefficient distortion, interaction loss, directional loss, and out-of-model leakage. On microphone-array data, compressions of comparable error differed in composition and in the tasks they harmed, and the composition predicted direction-finding damage far better than the error magnitude, though not on three other datasets.
\end{abstract}


\section{Introduction}

\IEEEPARstart{M}{ultidimensional} signals are compressed by low-rank approximation and judged by the reconstruction error. As a vector, a matrix, or a tensor, a signal is approximated by principal component analysis (PCA), the truncated SVD, or tensor decompositions \cite{kolda2009,sidiropoulos2017} such as the Tucker \cite{tucker1966}, CANDECOMP/PARAFAC (CP) \cite{carroll1970,harshman1970}, and tensor-train \cite{oseledets2011} decompositions, which underlie filtering, multisensor estimation, and reconstruction \cite{cichocki2015,muti2007,miron2020,sedighin2020}. The reconstruction error, however, is one scalar aggregating different losses: it does not say what was lost, and a smaller error need not mean a better result for perception \cite{blau2018,blau2019} or for a later task. Of two microphone-array compressions of comparable error, one raised the direction-of-arrival error by $43^\circ$, the other left it nearly unchanged but cost more class accuracy.

This letter asks how the representation shapes the reconstruction error. Rearranging the samples leaves the error of a given reconstruction unchanged but changes the constraints of low-rank approximation, a rank bound on one unfolding (matricized) or on every unfolding (Tucker). The answer has two parts. \emph{Magnitude dependence}: the representation sets the allowed approximations and hence the least achievable error. \emph{Composition dependence}: at the same error magnitude, even for one reconstruction, the error composition can differ with the representation, which can matter to a downstream task.

Existing results concern mainly the size of the error: the least error of a projection or of an optimal matrix truncation \cite{eckart1936} and bounds for the truncated higher-order SVD (HOSVD) \cite{delathauwer2000a,vannieuwenhoven2012}, each in one representation, and comparisons of matrix and tensor processing through the accuracy of chosen estimators \cite{haardt2008,roemer2014,signoretto2011,zhang2018}. Orthogonal splits of tensor-product spaces are classical in the analysis of variance \cite{takemura1983,lin2000}, and, unlike the singular values of a matrix, the $n$-mode singular values do not determine a tensor up to orthogonal transformations of its modes \cite{hackbusch2017a}. Composition dependence, and one account of both parts across vectors, matrices, and tensors, is to our knowledge missing.

Both parts are analyzed with one basis, the orthogonal decomposition of the squared reconstruction error, generalized to any order: the error of any reconstruction splits exactly into coefficient distortion, interaction loss, directional loss, and out-of-model leakage, recovering the known errors of vectors and matrices (Section~\ref{sec:split}). It shows that the error added by tighter nested constraints is exactly new loss (Theorem~\ref{thm:gap}, Section~\ref{sec:mag}), that one reconstruction $\F_{k,s}(\X)$ has an error of equal magnitude but different composition as a matrix and as a tensor (Proposition~\ref{prop:view}, Section~\ref{sec:comp}), and, as a consequence for tasks, that a component changes a differentiable task statistic only through the statistic's sensitivity to it (\eqref{eq:attr}, Proposition~\ref{prop:inv}, Section~\ref{sec:task}). On microphone-array data, compressions of comparable error differed in composition and in the tasks they damaged, and the composition predicted direction-finding damage far better than the error magnitude in a registered test, though not on three other datasets (Sections~\ref{sec:exp},~\ref{sec:disc}).

Throughout, $\X\in\mathbb{F}^{I_1\times\cdots\times I_K}$, $\mathbb{F}\in\{\mathbb{R},\mathbb{C}\}$, is a vector for $K=1$, a matrix for $K=2$, and a tensor for $K\ge3$; $\X_{(j)}$ is its mode-$j$ unfolding and $\times_j$ the mode-$j$ product \cite{kolda2009}, so that $\mathrm{vec}(\X\times_1\mathbf{M}_1\cdots\times_K\mathbf{M}_K)=(\mathbf{M}_K\otimes\cdots\otimes\mathbf{M}_1)\mathrm{vec}(\X)$ with $\otimes$ the Kronecker product; $\Pi_s(\cdot)$ is the rank-$s$ truncated SVD, and $\F_{k,s}(\X)$ the tensor obtained by applying $\Pi_s$ to $\X_{(k)}$.

\section{Generalized Decomposition of Reconstruction Error}\label{sec:split}

Both dependences are analyzed with one tool: we split the error energy into orthogonal parts whose subspaces tell where the lost content lies, outside the signal model or inside it and discarded in which modes. Let $\Xt$ be any reconstruction of $\X\ne0$, $\vx=\mathrm{vec}(\X)$, $\xt=\mathrm{vec}(\Xt)$, $\ve=\vx-\xt$ the error, and $\RE^2=\nrm{\ve}^2/\nrm{\vx}^2$ the relative squared reconstruction error.

\textbf{Model and residual.} Each mode $j$ has a model subspace, onto which an orthogonal projection matrix $\bP_j$ of rank $r_j$ projects, e.g., $\bP_j=\bU_j\bU_j^{\mathsf H}$ for a Tucker model $\X=\G\times_1\bU_1\cdots\times_K\bU_K$ \cite{tucker1966}. Then $\bP=\bP_K\otimes\cdots\otimes\bP_1$ projects onto the model subspace and $\bPi_{\mathrm{O}}=\Id-\bP$ onto its complement:
\[
\nrm{\ve}^2=\nrm{\bPi_{\mathrm{O}}\ve}^2+\nrm{\bP\ve}^2.
\]

\textbf{Retained and discarded directions.} In each mode, the reconstruction keeps some model directions and discards the rest: $\bP_j=\bQ_j+\bD_j$ with $\bP_j\bQ_j=\bQ_j$ and $\bQ_j\bD_j=\mathbf{0}$. Here $\bQ_j$ projects onto the left singular vectors of the mode-$j$ unfolding of the in-model reconstruction $\bP\xt$ whose squared singular values exceed $\tau\nrm{\bP\xt}^2$ (resolution $\tau\in[0,1)$), and $\bD_j=\bP_j-\bQ_j$ onto the discarded ones.

\textbf{Mode-wise blocks.} Expanding the product \cite{takemura1983,lin2000},
\begin{equation}\label{eq:resolve}
\bP=(\bQ_K+\bD_K)\otimes\cdots\otimes(\bQ_1+\bD_1)=\textstyle\sum_T\KK_T,
\end{equation}
where $T$ runs over the subsets of the modes and $\KK_T$ has the factor $\bD_j$ for $j\in T$ and $\bQ_j$ for $j\notin T$, e.g., $\KK_{\{1\}}=\bQ_3\otimes\bQ_2\otimes\bD_1$ for $K=3$; $\KK_T\ve$ is the in-model error discarded in exactly the modes in $T$. As $\bQ_j\bD_j=\mathbf{0}$, the $2^K$ matrices $\KK_T$ project onto mutually orthogonal subspaces. The \emph{partners} of content in mode $j$ are the directions it is coupled with in the other modes. Grouping the blocks by $|T|$, let
\[
\bPi_{\mathrm{C}}=\KK_\emptyset,\qquad\bPi_{\mathrm{MI}^{(j)}}=\KK_{\{j\}},\qquad\bPi_{\mathrm{D}}=\textstyle\sum_{|T|\ge2}\KK_T
\]
for $K\ge2$; for $K=1$, which has no partners, $\bPi_{\mathrm{D}}=\KK_{\{1\}}$ and there is no $\bPi_{\mathrm{MI}}$. Let $L_c=\nrm{\bPi_c\ve}^2/\nrm{\vx}^2$ and $\LMI=\sum_j\LMI^{(j)}$.

\begin{proposition}[Generalized decomposition]\label{prop:split}
For every reconstruction, model subspace, and resolution,
\begin{equation}\label{eq:split}
\RE^2=\LC+\textstyle\sum_{j}\LMI^{(j)}+\LD+\LO,
\end{equation}
the energies of the error in the mutually orthogonal subspaces of Table~\ref{tab:comp} (proof in the supplementary material).
\end{proposition}

It needs only $\X$ and $\Xt$, so it applies to any algorithm, and depends on the model subspace and on $\tau$, taken from $\X$ (e.g., its HOSVD) and shared by all compressions; for noisy data, it describes the departure from the observed $\X$.

\begin{table}[!b]
\centering
\caption{The Four Components of Squared Reconstruction Error}\label{tab:comp}
\footnotesize
\setlength{\tabcolsep}{2.8pt}
\renewcommand{\arraystretch}{1.15}
\begin{tabular}{@{}>{$}l<{$}>{\raggedright\arraybackslash}m{15.5mm}>{\raggedright\arraybackslash}m{16mm}>{\raggedright\arraybackslash}m{42mm}@{}}
\toprule
\multicolumn{2}{@{}l}{Component} & Subspace & Geometric interpretation\\
\midrule
\LC & Coefficient distortion & $\KK_\emptyset$ & Retained in all modes; wrong coefficients\\
\LMI^{(j)} & Interaction loss & $\KK_{\{j\}}$ & Discarded in mode $j$ only, partners retained: a coupling is lost\\
\LD & Directional loss & $\KK_T$, $|T|\ge2$ & Discarded in two or more modes\\
\LO & Out-of-model leakage & $\Id-\bP$ & Outside the model subspace\\
\bottomrule
\multicolumn{4}{@{}l@{}}{\scriptsize For $K=1$, the subspace of $\LD$ is $\KK_{\{1\}}$, and there is no $\LMI$.}
\end{tabular}
\end{table}

\section{Representation Dependence of Reconstruction Error}\label{sec:mvt}

Representations differ only in how the projection matrices factor; merging modes merges their factors, e.g., into $\bP_2\otimes\bP_1$. Let $\X=\G\times_1\bU_1\cdots\times_K\bU_K$ with orthonormal $\bU_j$ and $r_j=\rank\X_{(j)}$, take $\bP_j=\bU_j\bU_j^{\mathsf H}$ and $\tau=0$; proofs are in the supplementary material. A compression
\[
\xt=\widehat\bP\vx,\qquad\widehat\bP=\widehat\bQ_K\otimes\cdots\otimes\widehat\bQ_1,\quad\bP_j\widehat\bQ_j=\widehat\bQ_j,
\]
with orthogonal projection matrices $\widehat\bQ_j$ is \emph{projective}, and \emph{faithful} if each $\widehat\bQ_j$ projects onto the column space of $\Xt_{(j)}$, so that $\bQ_j=\widehat\bQ_j$ in the decomposition. The truncated SVD $\F_{k,s}$ of an unfolding, the truncated HOSVD, and the best multilinear-rank approximation are projective. $\F_{k,s}$ is faithful only under the \emph{rank condition}
\begin{equation}\label{eq:rank}
\textstyle s\prod_{l\notin\{j,k\}}r_l\ge r_j\quad\text{for all }j\ne k,
\end{equation}
and then for almost every core if $r_j\le\prod_{l\ne j}r_l$.

\subsection{Representation Dependence of Error Magnitude}\label{sec:mag}

The error $\ve$ of a given reconstruction, and hence $\RE^2$, is the same in every representation; the representation sets only the allowed approximations and hence the least achievable error. At equal bounds on the modes both constrain, tighter constraints cannot lower it and merging modes cannot raise it; at equal numbers of parameters the sets are not nested, and either representation may win. For nested projective compressions, the decomposition gives their cost:

\begin{theorem}\label{thm:gap}
Let $\xt_A=\widehat\bP_A\vx$ and $\xt_B=\widehat\bP_B\vx$ be faithful projective compressions with errors $\ve_A$, $\ve_B$ and nested retained subspaces, $\widehat\bQ{}^A_j\widehat\bQ{}^B_j=\widehat\bQ{}^B_j$ for all $j$, so that $B$ is the more constrained, and let $\KK^B_T$ be the blocks \eqref{eq:resolve} of $B$. Then $\LC=\LO=0$ for both, and:
(i) the extra error is the \emph{new loss} $\vd=\xt_A-\xt_B$, what $A$ keeps and $B$ discards:
\begin{equation}\label{eq:gap}
\RE_B^2-\RE_A^2=\nrm{\vd}^2/\nrm{\vx}^2;
\end{equation}
(ii) the directional loss of $A$ stays directional in $B$, and of the mode-$j$ interaction loss of $A$ only $\nrm{\KK^B_{\{j\}}\ve_A}^2/\nrm{\vx}^2$, whose partners $B$ keeps, stays interaction loss, the rest becoming directional;
(iii) of the new loss, $\nrm{\KK^B_{\{j\}}\vd}^2/\nrm{\vx}^2$ is mode-$j$ interaction loss of $B$ and the rest directional.
\end{theorem}

Geometrically, $\ve_A$ is orthogonal to the subspace $A$ keeps, which contains $\vd$, so $\ve_B=\ve_A+\vd$ gives \eqref{eq:gap} by Pythagoras. The supplementary material extends this to least achievable errors and to commuting, non-nested compressions.

\subsection{Representation Dependence of Error Composition}\label{sec:comp}

For a vector ($K=1$), a projection $\widehat\bQ_1$ with $\bP_1\widehat\bQ_1=\widehat\bQ_1$ gives $\RE^2=\LD$, the energy of the discarded directions, as in PCA. For a matrix ($K=2$), a truncated SVD $\Xt=\Pi_s(\X)$ retains the leading left and right singular vectors together, so every discarded component is discarded in both modes:
\begin{equation}\label{eq:mat}
\LMI=0\quad\text{and}\quad\RE^2=\LD,
\end{equation}
with $\LD\nrm{\vx}^2$ the Eckart--Young error \cite{eckart1936}; other matrix compressions can lose coupling. For a tensor ($K\ge3$), truncating one unfolding, $\widehat\bQ_k=\widehat\bQ$ and $\widehat\bQ_j=\bP_j$ for $j\ne k$, keeps the directions of $\widehat\bQ$ in mode $k$, so every block with $k\notin T$ vanishes and
\begin{equation}\label{eq:tens}
\LC=\LO=0\quad\text{and}\quad\RE^2=\LMI^{(k)}+\LD,
\end{equation}
where $\LD=0$ if and only if every mode $j\ne k$ keeps its entire subspace, as for a faithful $\F_{k,s}$. Orthogonally decomposable tensors give $\LMI=0$ (supplementary material).

\begin{proposition}[Same error, different composition]\label{prop:view}
Let $K\ge3$ and $\Xt=\F_{k,s}(\X)$, the rank-$s$ truncated SVD of the mode-$k$ matricization $\X_{(k)}$. With these model subspaces, $\Xt$ gives
\begin{align*}
\RE^2&=\LD && \text{read as a vector or a matrix,}\\
\RE^2&=\LMI^{(k)} && \text{read as a tensor ($\F_{k,s}$ faithful).}
\end{align*}
\end{proposition}

One reconstruction has one error magnitude in every representation but can have different compositions: a discarded mode-$k$ direction's partner is one right singular vector in the matrix, discarded with it, but a mix of other-mode directions in the tensor, kept by other content. Neither the error nor $n$-mode singular values tell them apart (supplementary material).

\subsection{Implications for Downstream Tasks}\label{sec:task}

Let $m$ be a real-valued task statistic of the signal, continuously differentiable near the segment from $\vx$ to $\xt$ (e.g., a classifier's margin), with gradient $\vg$, so that $m(\vx+\boldsymbol{\delta})\approx m(\vx)+\mathrm{Re}\{\vg(\vx)^{\mathsf H}\boldsymbol{\delta}\}$ for small $\boldsymbol{\delta}$. Integrating along the segment and writing $\ve=\sum_c\bPi_c\ve$,
\begin{equation}\label{eq:attr}
\textstyle m(\xt)-m(\vx)=-\sum_c\int_0^1\mathrm{Re}\{\vg(\vx-t\ve)^{\mathsf H}\bPi_c\ve\}\,dt,
\end{equation}
so a component changes $m$ only through $\bPi_c\vg$, the task's sensitivity to it, whereas the reconstruction error weighs every component alike. It also says when a larger error does no harm:

\begin{proposition}[Task invariance under increased error]\label{prop:inv}
Let $\xt_A$, $\xt_B$ be reconstructions, $\vd=\xt_A-\xt_B$, $\bPi$ an orthogonal projection matrix with $\bPi\vd=\vd$, and $m$ continuously differentiable near the segment between them. Then
\[
|m(\xt_B)-m(\xt_A)|\le\max_{0\le t\le1}\nrm{\bPi\vg(\xt_A-t\vd)}\,\nrm{\vd},
\]
and $m(\xt_A)=m(\xt_B)$ if $\mathrm{Re}\{\vg(\xt_A-t\vd)^{\mathsf H}\vd\}=0$ for $0\le t\le1$. For $\xt_A$, $\xt_B$ as in Theorem~\ref{thm:gap}, $\vd$ is the new loss, $\bPi=\widehat\bP_A-\widehat\bP_B$, and $\nrm{\vd}=\nrm{\vx}(\RE_B^2-\RE_A^2)^{1/2}$: the larger error leaves $m$ unchanged if $m$ is insensitive to the new loss.
\end{proposition}

These guarantees cover differentiable statistics only; a class label or a grid-searched direction is piecewise constant, so accuracy and DOA error are tested empirically (Section~\ref{sec:exp}).

\section{Experimental Results}\label{sec:exp}

We test the theory on simulated recordings of a four-microphone array, the TAU Spatial Sound Events 2019 dataset \cite{adavanne2019} (Fig.~\ref{fig:setup}): sound events of $11$ classes are convolved with impulse responses measured in five rooms for four microphones on a rigid sphere ($4.2$~cm radius), and the room's noise is added. One or two sources are active, each with its direction of arrival, $\mathrm{DOA}=(\phi,\theta)$, with azimuth $\phi$ and elevation $\theta$.

\begin{figure}[!h]
\centering
\includegraphics[width=\columnwidth]{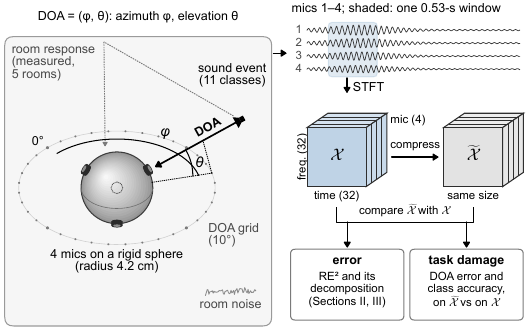}
\caption{The microphone-array data and the problem. Left: a simulated scene (schematic); $\mathrm{DOA}=(\phi,\theta)$ is estimated on a $10^\circ$ grid of both angles (dots: $\theta=0^\circ$). Right: a $0.53$-s window gives the STFT tensor $\X$; its compression $\Xt$ is compared with $\X$ by $\RE^2$, its decomposition, and two tasks.}\label{fig:setup}
\end{figure}

The problem is to compress a short multichannel spectrogram before it is analyzed: a $0.53$-s window gives the tensor $\X$ of short-time Fourier coefficients, $4$ microphones $\times$ $32$ frequencies ($0.125$--$4$~kHz) $\times$ $32$ frames. Each of $49$ compression configurations, matricized ($\F_{k,s}$ for every unfolding and several $s$) or tensor (truncated HOSVDs of five multilinear ranks and CP models of four ranks), returns $\Xt$ of the same size; $\RE^2$ and its decomposition compare $\Xt$ with $\X$, with the HOSVD of each window, of ranks $(4,32,32)$, as the model subspace. On $\Xt$ and on $\X$, the DOA is estimated by steered response power with phase transform (SRP-PHAT) \cite{dibiase2001}, steered by templates from the development set, which keeps only the phases of the cross-spectra $G_{mn}(f)=\sum_t\X[m,f,t]\,\overline{\X[n,f,t]}$ and picks the best-matching direction, and the class by logistic regression on microphone-averaged log power. The DOA error $\epsilon=\arccos\big(\mathbf{u}(\hat\phi,\hat\theta)^{\mathsf T}\mathbf{u}(\phi,\theta)\big)$, with $\mathbf{u}$ a unit direction vector, is the angle between estimated and true directions; the damage is $\Delta\mathrm{DOA}=\epsilon(\Xt)-\epsilon(\X)$ and $\Delta\mathrm{Accuracy}$ in percentage points (pp).

After an inconclusive check on the $400$ development recordings, we registered tests on the $100$ unopened evaluation recordings (supplementary material): C1--C3 and S2 pair $\F_{1,1}$ with the time truncation of smallest $\RE^2$ not below its own, chosen from the components alone; C4 and S1 concern the DOA damage of all $49$ configurations. The findings go from error composition to task damage to prediction (Fig.~\ref{fig:matched}).

\begin{figure*}[t]
\centering
\includegraphics[width=\textwidth]{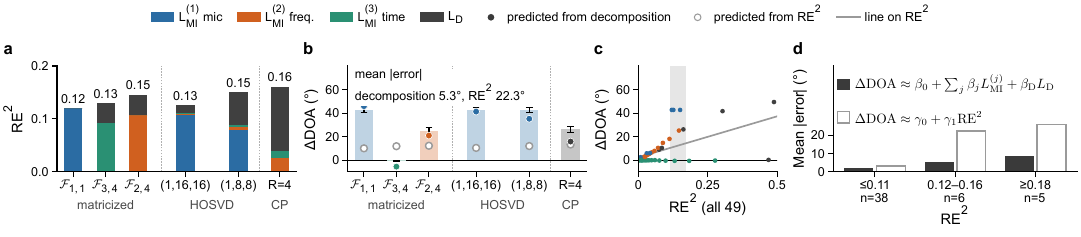}
\caption{Reconstruction errors and DOA damage (evaluation set, one source). (a), (b) The six configurations of $\RE^2$ $0.12$--$0.16$: matricized compressions (truncated SVD of one unfolding), truncated HOSVDs (ranks given), and a rank-4 CP model; $\F_{1,1}$ and $\F_{3,4}$ are the registered pair. (a) $\RE^2$ and its decomposition. (b) $\Delta\mathrm{DOA}$ (bars, means over 1,297 windows, colored by the dominant component; error bars, $95\%$ bootstrap intervals over $50$ recordings), and leave-one-out predictions \eqref{eq:loo} from the decomposition (filled) and $\RE^2$ (open). (c) $\Delta\mathrm{DOA}$ against $\RE^2$, all $49$; line, least-squares fit; shaded, range of (a), (b). (d) Mean absolute leave-one-out error of both models in three ranges of $\RE^2$ (exploratory).}\label{fig:matched}
\end{figure*}

\textbf{Comparable errors differ in composition.} By Proposition~\ref{prop:view}, an error magnitude does not fix its composition. The registered rule paired $\F_{1,1}$ with $\F_{3,4}$ ($\RE^2=0.120$ and $0.129$ in the one- and two-source scenes, against $0.131$ and $0.141$), and the decomposition, computed before any task was run, set them apart (Fig.~\ref{fig:matched}(a)). $\F_{1,1}$ meets the rank condition \eqref{eq:rank} ($1\cdot32\ge32$ in both other modes), so it is faithful and lost only microphone coupling, $\RE^2=\LMI^{(1)}$. $\F_{3,4}$ violates it in the frequency mode ($4\cdot4<32$) and lost mostly time coupling, with $29$ and $25\%$ of its error in $\LD$. $\F_{2,4}$ lost mostly frequency coupling, the HOSVDs of microphone rank one mostly microphone coupling, and the CP model, constraining every mode, mostly directional loss ($0.121$ of $0.162$). Tighter constraints changed magnitude and composition: post hoc, the HOSVD of ranks $(1,8,8)$, keeping the microphone subspace of $\F_{1,1}$, raised $\RE^2$ from $0.120$ to $0.150$; as Theorem~\ref{thm:gap} predicts, the added $0.030$ was new loss, and $34\%$ of the microphone interaction loss of $\F_{1,1}$ became $\LD$.

\textbf{Different compositions damage different tasks.} Without compression, the DOA errors were $3.3^\circ$ and $31.9^\circ$ for one and two sources (chance: $90^\circ$, $73^\circ$) and the class accuracy $64\%$. $\F_{1,1}$ raised the DOA error by $42.6^\circ$ and $27.4^\circ$, $\F_{3,4}$ by $-0.6^\circ$ and $2.4^\circ$ (differences $43.3^\circ$ $[41.0,45.5]$ and $24.9^\circ$ $[22.5,27.3]$, $p=10^{-4}$; C1, C2, that $\F_{1,1}$ raises it more, held; Fig.~\ref{fig:matched}(b)), whereas $\Delta\mathrm{Accuracy}$ was $-4.9$~pp under $\F_{3,4}$ and $-1.5$~pp under $\F_{1,1}$, and the log-probability of the true class fell by $0.21$ $[0.17,0.25]$ nats more under $\F_{3,4}$ (S2, that the time truncation harms the class more, held). The DOA lies in the coupling of the microphone and frequency modes: for one source $\X[m,f,t]\approx A[m,f]\,S[f,t]$, so
\[
G_{mn}(f)\approx A[m,f]\,\overline{A[n,f]}\textstyle\sum_t|S[f,t]|^2,
\]
whose phase grows with frequency at a rate set by the DOA. A time projection scales every $G_{mn}(f)$ by a factor in $[0,1]$, leaving SRP-PHAT unchanged if the factor never vanishes, whereas a microphone-rank-one reconstruction $u_mY[f,t]$ has frequency-independent phases $\arg u_m\bar u_n$ and loses the DOA. With two sources, a time projection also changes the Gram matrix of the two waveforms at each frequency, and C3, a rise below $2^\circ$, failed ($2.4^\circ$ $[1.1,3.8]$). Post hoc, the microphone-rank-one HOSVDs, keeping the $\mathbf{u}$ of $\F_{1,1}$, gave its DOA error in every window, so the damage follows the lost coupling, not matricization; $\F_{2,4}$ and the CP model, with four frequency directions, raised it for one source by $25.1^\circ$ and $26.1^\circ$ and cost the most accuracy ($-11.1$, $-13.0$~pp).

\textbf{Composition predicts DOA damage better than error magnitude.} The magnitude did not track the damage: in Fig.~\ref{fig:matched}(b), the smallest $\RE^2$ did the largest DOA damage; for one source, no time truncation, down to $\F_{3,1}$ at $\RE^2=0.47$, changed the DOA error by more than $0.7^\circ$ (Fig.~\ref{fig:matched}(c)), and the HOSVD of ranks $(1,8,8)$ added $25\%$ to the $\RE^2$ of $\F_{1,1}$, all new loss, without changing it; from $\F_{1,2}$ to the HOSVD of ranks $(2,8,8)$, in contrast, the DOA damage rose from $5.7^\circ$ to $10.3^\circ$. For configuration $c$, let $y_c$ be the mean $\Delta\mathrm{DOA}$ over the windows and $\mathbf{a}_c$ the mean of $(1,\LMI^{(1)},\LMI^{(2)},\LMI^{(3)},\LD)$, or of $(1,\RE^2)$. A linear model and its leave-one-out $R^2$ are
\begin{equation}
y_c\approx\mathbf{a}_c^{\mathsf T}\boldsymbol\beta,\qquad
R^2=1-\frac{\sum_c\big(y_c-\mathbf{a}_c^{\mathsf T}\hat{\boldsymbol\beta}_{-c}\big)^2}{\sum_c(y_c-\bar y)^2},\label{eq:loo}
\end{equation}
where $\hat{\boldsymbol\beta}_{-c}$ is the least-squares fit to the other $48$ configurations and $\bar y$ the mean of all $y_c$. C4, that the decomposition predicts the damage better than $\RE^2$, held: $R^2$ was $0.87$ and $0.86$ against $0.09$ and $0.13$ (class: $0.95$ against $0.92$), and the mean absolute error $5.3^\circ$ against $22.3^\circ$ at comparable $\RE^2$ but $2.1^\circ$ against $3.1^\circ$ below it (both exploratory; Fig.~\ref{fig:matched}(b), (d)). S1, $|\beta_3|<\beta_1/10$ (coefficients of $\LMI^{(3)}$, $\LMI^{(1)}$), failed, as $\F_{3,1}$--$\F_{3,3}$ carry up to $0.31$ of equally harmless $\LD$, which one coefficient cannot express; the hypothesis is partly supported.

\section{Discussion}\label{sec:disc}

Both parts rest on Proposition~\ref{prop:split} but differ in scope. Magnitude dependence, of the least achievable error and not of a given reconstruction's error, holds for every representation. Composition dependence is shown to be possible, not universal; Theorem~\ref{thm:gap} and Proposition~\ref{prop:view} assume data in the model and faithful projective compressions, exact for noisy data read as observed and approximate for a clean signal. Whether composition matters depends on the task: \eqref{eq:attr} and Proposition~\ref{prop:inv} assume a differentiable statistic, and the experiments cover one dataset and two non-differentiable tasks; on three other datasets the decomposition predicted damage no better (supplementary material), where, post hoc, tasks were sensitive to all components, truncations violated \eqref{eq:rank}, large losses removed whole sources, or the predictor was retrained. These conditions are not necessary; each task must be checked.

\section{Conclusion}

We asked how the representation shapes the reconstruction error. It sets the allowed approximations and so the least achievable error, not the error of a given reconstruction, and it can change the composition of an error of equal magnitude, which decided which task was damaged. One decomposition analyzes both and predicted task damage better than the error on one of four datasets. Minimizing the error is not minimizing the loss of what a task needs; noise and non-projective compressions such as CP remain open.

\section*{Acknowledgment}

The author declares no conflict of interest. Claude (Anthropic) assisted with English proofreading, revising the text, and preparing the figures and tables.


\clearpage
\twocolumn[{\centering{\LARGE Supplementary Material}\par\vspace{1.2em}}]
\markboth{Manuscript under review (supplementary material)}{Hasegawa: How Vector, Matrix, and Tensor Representations Shape Reconstruction Error}
\setcounter{section}{0}

\section{Proofs}

Matrices act on tensors through vectorization; $\E=\X-\Xt$. Proofs follow the letter: decomposition, magnitude, composition, tasks (Sections~\ref{sec:split}, \ref{sec:mag}--\ref{sec:task}).

\begin{lemma}\label{lem:part}
For each mode $j$, let $\mathbf{A}_j^{(1)},\dots,\mathbf{A}_j^{(m)}$ be orthogonal projection matrices with $\mathbf{A}_j^{(a)}\mathbf{A}_j^{(b)}=\mathbf{0}$ for $a\ne b$, and $\mathbf{C}_j=\sum_a\mathbf{A}_j^{(a)}$. Then the $m^K$ Kronecker products $\mathbf{A}_\sigma=\mathbf{A}_K^{(\sigma_K)}\otimes\cdots\otimes\mathbf{A}_1^{(\sigma_1)}$ are orthogonal projection matrices onto mutually orthogonal subspaces and sum to $\mathbf{C}_K\otimes\cdots\otimes\mathbf{C}_1$.
\end{lemma}
\begin{IEEEproof}
By the mixed-product rule, $\mathbf{A}_\sigma$ is Hermitian and idempotent, $\mathbf{A}_\sigma\mathbf{A}_{\sigma'}$ has a factor $\mathbf{A}_j^{(a)}\mathbf{A}_j^{(b)}=\mathbf{0}$ for $\sigma\ne\sigma'$, and expanding the product of the sums $\mathbf{C}_j$ gives the sum.
\end{IEEEproof}

\begin{IEEEproof}[Proof of Proposition~\ref{prop:split}]
Lemma~\ref{lem:part} with $m=2$ and $(\bQ_j,\bD_j)$ gives $\sum_T\KK_T=\bP$; with $\Id-\bP$, the $2^K+1$ matrices project onto mutually orthogonal subspaces and sum to $\Id$, and \eqref{eq:split} follows from the Pythagorean theorem.
\end{IEEEproof}

\begin{IEEEproof}[Proof of Theorem~\ref{thm:gap} (Section~\ref{sec:mag})]
Faithfulness gives $\bQ_j=\widehat\bQ{}^A_j$ in the decomposition of $A$ and $\bQ_j=\widehat\bQ{}^B_j$ in that of $B$. Lemma~\ref{lem:part} with $m=3$ and the projection matrices $\bP_j-\widehat\bQ{}^A_j$, $\widehat\bQ{}^A_j-\widehat\bQ{}^B_j$, $\widehat\bQ{}^B_j$ (orthogonal as $\bP_j\widehat\bQ{}^A_j=\widehat\bQ{}^A_j$ and $\widehat\bQ{}^A_j\widehat\bQ{}^B_j=\widehat\bQ{}^B_j$) divides $\vx=\bP\vx$ into $3^K$ orthogonal cells $\vx_{S,R}$, discarded by both in the modes $S$ and by $B$ alone in the modes $R$. As $\xt_A=\widehat\bP_A\vx$ is the sum of the cells with $S=\emptyset$ and $\xt_B$ the cell $(\emptyset,\emptyset)$, $\ve_A=\sum_{S\ne\emptyset}\vx_{S,R}$ and $\vd=\sum_{R\ne\emptyset}\vx_{\emptyset,R}$ are orthogonal and $\ve_B=\ve_A+\vd$, which gives \eqref{eq:gap} (without faithfulness). The block $\KK^A_T$ keeps the cells with $S=T$, and $\KK^B_T$ those with $S\cup R=T$. Hence $\KK^A_\emptyset\ve_A=\KK^B_\emptyset\ve_B=\mathbf{0}$ and, $\vx$ being in the model, $\LC=\LO=0$; $\KK^B_{\{j\}}\ve_A=\vx_{\{j\},\emptyset}$; $\KK^B_{\{j\}}\vd=\vx_{\emptyset,\{j\}}$; and $\KK^A_{\{j\}}\ve_A=\sum_R\vx_{\{j\},R}$. The other cells of $\ve_B$, with $|S\cup R|\ge2$, are directional loss of $B$.
\end{IEEEproof}

Least achievable errors: for the best approximation $\X^*=\X\times_j\bQ^*_j$ of multilinear ranks $(s_j)$ and $\Xt'=\X\times_k\bQ^*_k$, $\RE^2(\X^*)-\RE^2(\F_{k,s})$ is \eqref{eq:gap} for the nested pair $(\Xt',\X^*)$ plus $\RE^2(\Xt')-\RE^2(\F_{k,s})\ge0$ (Eckart--Young), a difference of $\LMI^{(k)}$ as both are faithful.

Commuting compressions ($\widehat\bQ{}^A_j\widehat\bQ{}^B_j=\widehat\bQ{}^B_j\widehat\bQ{}^A_j$): Lemma~\ref{lem:part} with $m=4$ gives cells that each keeps or discards whole, and $\RE_B^2-\RE_A^2$ is the energy kept by $A$ in every mode and lost by $B$, minus the converse.

\begin{IEEEproof}[Proofs for Section~\ref{sec:comp}]
In all parts $\X$ lies in the model subspace and $\Xt$ is \emph{projective}, $\Xt=\X\times_1\widehat\bQ_1\cdots\times_K\widehat\bQ_K$ with $\bP_j\widehat\bQ_j=\widehat\bQ_j$; hence $\LO=0$. For any $\tau$, $\bQ_j$ projects onto a subspace of the column space of $\Xt_{(j)}$, which $\widehat\bQ_j$ leaves fixed, so $\bQ_j\widehat\bQ_j=\bQ_j$, $\KK_\emptyset\Xt=\X\times_j\bQ_j\widehat\bQ_j=\KK_\emptyset\X$, and $\LC=0$.

Matrix case \eqref{eq:mat}: with $\X\times_1\mathbf{M}_1\times_2\mathbf{M}_2=\mathbf{M}_1\mathbf{X}\mathbf{M}_2^\top$ and $\Xt=\Pi_s(\mathbf{X})$, $\bQ_1$ and $\bQ_2$ project onto subspaces of the column spaces of $\Xt$ and $\Xt^\top$, orthogonal to those of $\mathbf{E}$ and $\mathbf{E}^\top$; so $\bQ_1\mathbf{E}=\mathbf{0}$, $\mathbf{E}\bQ_2^\top=\mathbf{0}$, $\KK_{\{1\}}\E=\bD_1\mathbf{E}\bQ_2^\top=\mathbf{0}$, $\KK_{\{2\}}\E=\bQ_1\mathbf{E}\bD_2^\top=\mathbf{0}$, and $\RE^2=\LD$ for any $\tau$.

Tensor case \eqref{eq:tens}: let $\bU_{-k}=\bigotimes_{l\ne k}\bU_l$, so $\X_{(k)}=\bU_k\G_{(k)}\bU_{-k}^\top$, and $\widehat\bQ=\bU_k\mathbf{\Phi}\mathbf{\Phi}^{\mathsf H}\bU_k^{\mathsf H}$ with orthonormal $\mathbf{\Phi}\in\mathbb{F}^{r_k\times s}$; then $\Xt=(\mathcal{C}\times_k\mathbf{\Phi})\times_1\bU_1\cdots\times_K\bU_K$ with $\mathcal{C}=\G\times_k\mathbf{\Phi}^{\mathsf H}$. As $\G_{(k)}$ has full row rank, $\Xt_{(k)}=\widehat\bQ\X_{(k)}$ has the column space of $\widehat\bQ$, so $\bQ_k=\widehat\bQ$ and $\X\times_k\bQ_k=\Xt$ at $\tau=0$. For $T\neq\emptyset$ with $k\notin T$ and $j\in T$, $\KK_T\X$ contains the factor $\Xt\times_j\bD_j=0$, and $\KK_T\Xt=0$; so $\LMI=\LMI^{(k)}$ and $\RE^2=\LMI^{(k)}+\LD$. By Lemma~\ref{lem:part}, $\sum_{T\ni j}\KK_T\X=\X\times_j\bD_j$, whose nonzero terms for $j\neq k$ have $|T|\ge2$; so $\LD=0$ iff $\bQ_j\X_{(j)}=\X_{(j)}$ for all $j\ne k$, i.e., as the columns of $\X_{(j)}$ span the subspace of $\bP_j$ and $\bP_j\bQ_j=\bQ_j$, iff $\bQ_j=\bP_j$: $\Xt_{(j)}$ and $\X_{(j)}$ have the same column space. Proposition~\ref{prop:view} follows, as the right singular vectors $\mathbf{v}$ of $\X_{(k)}$ satisfy $(\bigotimes_{j\ne k}\bP_j)\mathbf{v}=\mathbf{v}$, so \eqref{eq:mat} applies to the matrix, and a vector has no partners.

Rank condition \eqref{eq:rank}: take first i.i.d.\ $\mathcal{N}(0,1)$ core entries; the cores with a rank-deficient $\G_{(j)}$ or a repeated eigenvalue of $\G_{(k)}\G_{(k)}^\top$ form a null set (zeros of nonzero polynomials). For $\F_{k,s}$, $\mathbf{\Phi}$ and $\mathbf{\Psi}_s$ are the first $s$ columns of $\mathbf{\Phi}_0$ and $\mathbf{\Psi}$ in a compact SVD $\G_{(k)}=\mathbf{\Phi}_0\mathbf{\Sigma}\mathbf{\Psi}^\top$, and $\mathcal{C}_{(k)}=\mathbf{\Sigma}_s\mathbf{\Psi}_s^\top$; mode products with orthonormal columns preserve unfolding ranks, so $\rank\Xt_{(j)}=\rank\mathcal{C}_{(j)}$ depends only on the column space $\mathcal{V}$ of $\mathbf{\Psi}_s$. As $\G_{(k)}\mathbf{O}$ and $\G_{(k)}$ have the same law for every orthogonal $\mathbf{O}$, $\mathcal{V}$ is distributed as the column space of an $r_{-k}\times s$ Gaussian $\mathbf{V}$. The sum of the squared $r_j\times r_j$ minors of $(\mathrm{fold}_k\mathbf{V}^\top)_{(j)}$ is a polynomial, not identically zero when $s\prod_{l\notin\{j,k\}}r_l\ge r_j$ (take $(\mathrm{fold}_k\mathbf{V}^\top)_{(j)}=[\bI_{r_j},\mathbf{0}]$). Hence $\rank\Xt_{(j)}=r_j$ almost surely for $j\neq k$, $\LD=0$, and $\RE^2>0$ as $s<r_k$. The exceptional cores are Gaussian-null, hence null for every absolutely continuous law; complex data: circular Gaussian core, unitary $\mathbf{O}$, squared moduli of minors.

Odeco tensors, $\X=\sum_ic_i\mathbf{u}_i^{(1)}\circ\cdots\circ\mathbf{u}_i^{(K)}$ with $c_1>c_2>\cdots>0$ and orthonormal $\{\mathbf{u}_i^{(j)}\}_i$: $\X_{(k)}=\sum_ic_i\mathbf{u}_i^{(k)}\mathbf{w}_i^\top$ is an SVD, $\Xt=\sum_{i\le s}c_i\mathbf{u}_i^{(1)}\circ\cdots\circ\mathbf{u}_i^{(K)}$, $\bQ_j=\sum_{i\le s'}\mathbf{u}_i^{(j)}\mathbf{u}_i^{(j)\mathsf H}$ with a common $s'\le s$ for any $\tau$, and every error term lies in $\KK_{[K]}$, so $\LMI=0$.
\end{IEEEproof}

Equal reconstruction error and $n$-mode singular values, different composition: with $(c_1,c_2)=(1,0.8)$, the odeco core $(g_{111},g_{222})=(c_1,c_2)$ and the symmetric core $g_{111}=(c_1^2-c_2^2/2)^{1/2}$, $g_{122}=g_{212}=g_{221}=c_2/\sqrt2$ ($2\times2\times2$) both have $\G_{(j)}\G_{(j)}^\top=\mathrm{diag}(c_1^2,c_2^2)$ in every mode, and $\F_{1,1}$ gives $\RE^2=0.390$ for both, with $\LMI=0$ for the first and, keeping a rank-two slice, $\LMI=\RE^2$ for the second.

\begin{IEEEproof}[Proof of Proposition~\ref{prop:inv} (Section~\ref{sec:task})]
$\varphi(t)=m(\xt_A-t\vd)$ has $\varphi'(t)=-\mathrm{Re}\{\vg^{\mathsf H}\vd\}=-\mathrm{Re}\{(\bPi\vg)^{\mathsf H}\vd\}$, as $\bPi=\bPi^{\mathsf H}$ and $\bPi\vd=\vd$; integrate over $[0,1]$ and apply the Cauchy--Schwarz inequality. For nested subspaces, $\widehat\bP_A-\widehat\bP_B$ is an orthogonal projector, and \eqref{eq:gap} gives $\nrm{\vd}$.
\end{IEEEproof}

\clearpage
\twocolumn[{\centering{\LARGE Information for Reproducibility}\par\vspace{1.2em}}]

The microphone-array data are TAU Spatial Sound Events 2019 \cite{S-adavanne2019}: $400$ development and $100$ evaluation recordings, resampled to $16$~kHz; STFT with a $512$-point periodic Hann window and hop $256$; $32$ frames ($0.53$~s) and the bins $4,8,\dots,128$ ($125$~Hz to $4$~kHz) give $\X\in\mathbb{C}^{4\times32\times32}$. One-source windows start at the onset of an event with no other event in the window; two-source windows at the later onset of two overlapping events that stay active, with no third. DOA: the phase-transformed cross-spectra of the six microphone pairs are correlated with templates (the mean phase-transformed cross-spectra of the one-source development windows of each direction) for $324$ directions (azimuth $-180^\circ$ to $170^\circ$, elevation $-40^\circ$ to $40^\circ$, step $10^\circ$); for two sources, the two highest scores at least $30^\circ$ apart, scored under the better assignment. Class: log power averaged over microphones, standardized, multinomial logistic regression with $C=0.01$, trained on the development windows. Model subspace: per-window HOSVD at numerical rank (singular values above $10^{-10}$ times the largest), $(4,32,32)$ in every window; components at $\tau=0$. Compression configurations: $\F_{1,s}$ ($s=1,2,3$); $\F_{2,s}$ and $\F_{3,s}$ ($s=4,\dots,16,18,20,24,28$, and $s=1,2,3$ for $\F_{3,s}$); truncated HOSVD of ranks $(1,8,8)$, $(1,16,16)$, $(2,8,8)$, $(2,16,16)$, $(3,16,16)$; CP of ranks $1$, $2$, $4$, $8$ ($40$ alternating least-squares sweeps). Tests: C1--C3 by one-sided recording-cluster bootstrap ($10^4$ resamples) with Holm's correction at $\alpha=0.05$; C4 by resampling recordings $2000$ times, holding if the $2.5\%$ quantile of the difference of the leave-one-out $R^2$ is positive in both scenes. On the development set, the registered rule paired $\F_{1,1}$ ($\RE^2=0.116$) with $\F_{3,6}$ ($0.069$), outside the registered $\pm20\%$, so that check was inconclusive.

The synthetic tensors are $\X=\sum_{r=1}^3c_r\mathbf{a}(\mu_r)\circ\mathbf{b}(\nu_r)\circ\mathbf{s}_r\in\mathbb{C}^{12\times12\times12}$ (array steering vectors at angles at least half a Rayleigh unit apart, distinct columns of the unitary $12$-point DFT, unit-norm complex Gaussian waveforms, $|c_r|=1$). Labels: the bin of the median Doppler (presence, four bins) and the angle bin of the source with the largest Doppler (coupling, six bins), each read by a multilayer perceptron (MLP) $3456\to256\to256\to C$ trained on uncompressed tensors at $20$~dB (three seeds; clean accuracies $0.96$ and $0.87$). Seventeen compression configurations ($\F_{k,s}$ with $k\le3$, $s\le2$, truncated HOSVD, tensor-train SVD, CP) on $3000$ test tensors; the comparable-$\RE^2$ groups and the predictions were fixed from $600$ independent tensors before any accuracy was computed.

The electroencephalograms (EEG) are BCI Competition IV dataset 2a \cite{S-tangermann2012} (nine subjects, $22$ channels, four motor-imagery classes): nine bands ($4$--$40$~Hz) of $0.5$--$4.0$~s after the cue at $125$~Hz, $\X\in\mathbb{R}^{9\times22\times438}$; model subspace: per-trial HOSVD keeping $0.9999$ of the energy of every unfolding (time ranks $89$--$166$); shallow convolutional network (CNN) \cite{S-schirrmeister2017} per subject and seed (three), clean accuracy $0.58$; nine compression configurations with $\RE^2$ near $0.15$ or $0.35$. Gradient shares: $\nrm{\bPi_{\mathrm{MI}^{(j)}}\vg}/\nrm{\vg}$ on clean input, $m$ the margin of the true class. Baseband re-run: each band demodulated at its center and decimated by $32$, $\X\in\mathbb{C}^{9\times22\times28}$, $32$ compression configurations, clean accuracy $0.43$.

The hyperspectral image is AVIRIS Indian Pine Test Site~3 \cite{S-baumgardner2015} ($145\times145$ pixels, $200$ bands after removing water absorption); model subspace: HOSVD keeping $0.9999$ of the energy; spectral PCA, spatial truncations, and Tucker models with at most $1$, $2$, and $5\%$ of the entries; support vector machine (SVM) with radial basis kernel, retrained on $10\%$ of each class of each reconstruction (five splits), clean accuracy $83.1\%$. The components were bounded but not resolved, because no calibrated resolution separated weak directions from noise.

On these data (Table~\ref{tab:other}), comparable errors did different damage in several cases, but the decomposition never predicted it better than the reconstruction error, and each case missed one condition of Section~\ref{sec:disc} (post hoc reading): (i) no mode had a gradient share near zero for the EEG network ($0.23$, $0.10$, $0.20$), and the synthetic coupling label uses both modes it couples; (ii) the EEG band and channel truncations violated \eqref{eq:rank} and produced mostly $\LD$; (iii) at $\RE^2\approx0.63$, interaction loss shrank two of three synthetic sources to residues and acted as removal; (iv) spatial truncations, presumably by smoothing noise and within-class variation, raised the accuracy of the retrained SVM despite a larger $\RE^2$.

\begin{table}[h]
\centering
\caption{Registered tests on other data. Damage at comparable $\RE^2$: drop of accuracy in points (synthetic: $\F_{1,2}$, $\F_{2,2}$, CP 2 at $\RE^2\approx0.3$; EEG: channel minus band truncation; hyperspectral: accuracy at $\RE^2=0.090$ and $0.029$). AUC: per-sample errors of left-out configurations predicted by the decomposition / by $\RE^2$. $^\dagger$Post hoc; hyperspectral: registered for another question.}\label{tab:other}
\footnotesize
\setlength{\tabcolsep}{3pt}
\begin{tabular}{@{}lllll@{}}
\toprule
Data, task & Predictor & Missing & Damage & AUC\\
\midrule
Synth., presence & MLP & (iii) & $2.6$, $4.3$, $45.1$ & $0.60$ / $0.66$\\
Synth., coupling & MLP & (i) & $14.5$, $16.1$, $18.8$ & $0.72$ / $0.70$\\
EEG & CNN & (i), (ii) & $22.5^\dagger$ & $0.60$ / $0.56$\\
EEG, baseband & CNN & (i) & as $\RE^2$ & $0.711$ / $0.714$\\
Hyperspectral & SVM & (iv) & $84.6$, $58.2$ & neither\\
\bottomrule
\end{tabular}
\end{table}

In all experiments, the identity \eqref{eq:split} held within $4.9\times10^{-15}$.


\end{document}